\documentclass{article}

\usepackage[numbers]{natbib}
\usepackage{arxiv}
\usepackage{microtype}
\usepackage{graphicx}
\usepackage{subcaption}
\usepackage{booktabs} 
\usepackage{hyperref}

\usepackage[utf8]{inputenc}
\usepackage[T1]{fontenc}
\usepackage{booktabs}
\usepackage{amsfonts,amsmath,amsthm,amssymb}
\usepackage{nicefrac}
\usepackage{microtype}
\usepackage{mathtools}
\usepackage{enumitem}
\usepackage{xspace}
\usepackage{hyperref}
\usepackage{color}
\definecolor{darkblue}{rgb}{0.0,0.0,0.5}
\hypersetup{
	colorlinks = true,
	citecolor  = darkblue,
	linkcolor  = darkblue,
	citecolor  = darkblue,
	filecolor  = darkblue,
	urlcolor   = darkblue,
}
\usepackage{wrapfig}
\usepackage{tikz}
\usetikzlibrary{arrows}
\usepackage{breqn}

\tikzset{
  treenode/.style = {align=center, inner sep=4pt, text centered,
    font=\scshape\small},
  interm/.style = {treenode, black, draw=white,
    fill=white,minimum height=1em},
  leaf/.style = {treenode, black,
    font=\footnotesize\it}
}

\theoremstyle{definition}
\newtheorem{defn}{Definition}
\newtheorem{rmk}{Remark}

\newcommand{\ts}[3]{#1_{#2}^{#3}}
\newcommand{\tsi}[4]{#1_{#3}^{#4}(#2)}
\newcommand{\disc}{\ensuremath{\Delta}}
\newcommand{\error}{\mathcal L}
\newcommand{\expect}{\mathbb E}
\newcommand{\nextdist}{\ensuremath{\mathcal D}\xspace}
\newcommand{\dist}{\ensuremath{\mathcal D'}\xspace}
\newcommand{\radem}{\ensuremath{\mathfrak R}}
\newcommand{\hphi}{\ensuremath{\Phi_{\text{\scriptsize s2s}}}}
\newcommand{\nhphi}{\ensuremath{\Phi_{\text{\scriptsize loc}}}}

\newcommand{\hybloss}{\ensuremath{\error_{\text{\scriptsize hyb}}}}
\newcommand{\nhloss}{\ensuremath{\error_{\text{\scriptsize loc}}}}
\newcommand{\hbeta}{\ensuremath{\beta_{\text{{\scriptsize s2s}}}}}

\newcommand{\nlsum}{\ensuremath{\sum\nolimits}}
\newcommand{\train}{\ensuremath{{\bf Z}}\xspace}
\newcommand{\past}{\ensuremath{{\bf Y}}\xspace}
\newcommand{\tpast}{\ensuremath{{\bf Y}'}\xspace}
\newcommand{\hybphi}{\ensuremath{\Phi_{\text{\scriptsize hyb}}}}

\theoremstyle{plain}

\newtheorem{theorem}{Theorem}
\newtheorem{lemma}{Lemma}

\newtheorem{prop}{Proposition}

\makeatletter
\theoremstyle{plain}
\newtheorem*{rep@theorem}{\rep@title}
\newcommand{\newreptheorem}[2]{%
\newenvironment{rep#1}[1]{%
 \def\rep@title{#2 \ref{##1}}%
 \begin{rep@theorem}}%
 {\end{rep@theorem}}}
\makeatother

\newreptheorem{theorem}{Theorem}
\newreptheorem{lemma}{Lemma}
\newreptheorem{corollary}{Corollary}
\newreptheorem{proposition}{Proposition}

\numberwithin{equation}{section}
\numberwithin{theorem}{section}

\newcommand{\ignore}[1]{}

\newif\ifarxiv
\newif\ificml

\arxivtrue
\icmlfalse

\title{Foundations of Sequence-to-Sequence Modeling for Time Series}
\author{\name Vitaly Kuznetsov \email{vitalyk@google.com}\\
  \addr{Google Research, New York, NY} \\
  \name Zelda Mariet\thanks{Authors are in alphabetical order.} \email{zelda@csail.mit.edu}\\
  \addr{Massachusetts Institute of Technology, Cambridge, MA}
}

\date{}

\begin{document}
\maketitle

\begin{abstract}
  The availability of large amounts of time series data, paired with the performance of deep-learning algorithms on a broad class of problems, has recently led to significant interest in the use of sequence-to-sequence models for time series forecasting. We provide the first theoretical analysis of this time series forecasting framework. We include a comparison of sequence-to-sequence modeling to classical time series models, and as such  our theory can serve as a quantitative guide for practitioners choosing between different modeling methodologies.
\end{abstract}

\section{Introduction}
Time series analysis is a critical component of real-world applications such as
climate modeling, web traffic prediction, neuroscience,
as well as economics and finance. We focus on the fundamental
question of time series forecasting. Specifically, we study the task of
forecasting the next $\ell$ steps of an
$m$-dimensional time series $\past$, where $m$ is assumed to be very large. 
For example, in climate modeling, $m$ may correspond to the number of locations
at which we collect historical observations, and more generally to the number of sources which provide us with time series.
 
Often, the simplest way to tackle this problem is to approach it as
$m$ separate tasks, where for each of the $m$ dimensions
we build a model to forecast the univariate time series corresponding to that dimension.
Auto-regressive and state-space models~\citep{engle1982, bollerslev1986,brockwell1986,box1990,hamilton1994}, as well as non-parametric approaches
such as RNNs~\citep{bianchi17}, are often used in this setting.
To account for correlations between different time series,
these models have also been generalized
to the multivariate case
\citep{Lutkepohl06,lutkepohl2007,song2011,han2015,han2015b,banbura2010,basu2015,song2011,basu2015,sun2015,negahban2011,yu2016,lv2015}.
In both univariate and multivariate settings, an observation at time $t$ is treated as
a single sample point, and the model tries to capture relations between
observations at times $t$ and $t+1$. Therefore, we refer to these models
as \emph{local}.

In contrast, an alternative methodology based on treating $m$ univariate time series as $m$ samples drawn from some unknown distribution has also gained popularity in recent years. In this setting, each of the $m$ dimensions of $\past$ is treated as a separate example and a single model is learned from these $m$ observations. Given $m$ time series of length $T$, this model learns to map past vectors of length $T-\ell$ to corresponding future vectors of length $\ell$. LSTMs and RNNs~\citep{hochreiter1997}
are a popular choice of model class for this setup~\citep{deepar,GoelMylnykBanerjee2017,YuZhengAnandkumarYue2017,LiYuShahabiLiu2017,laptev2017,zhu2017}.\footnote{Sequence-to-sequence models are also among the winning solutions
in the recent time series forecasting competition: \tiny{\url{https://www.kaggle.com/c/web-traffic-time-series-forecasting}}.}
Consequently,
we refer to this framework as \emph{sequence-to-sequence modeling}.


While there has been progress in understanding the generalization ability of
local models~\citep{yu94,meir2000,mohri2009,mohri2010,kuznetsov14,kuznetsov2015,kuznetsov16,KuznetsovMohri17,zimin2017}, to the best of our knowledge the generalization
properties of sequence-to-sequence modeling have not yet been studied, raising the following natural questions:
\begin{itemize}
  \item What is the generalization ability of sequence-to-sequence models and
  how is it affected by the statistical properties of the underlying
  stochastic processes (e.g. non-stationarity, correlations)?
\item When is sequence-to-sequence modeling preferable to local modeling,
  and vice versa?
\end{itemize}
We provide the first generalization guarantees for time series forecasting with sequence-to-sequence models. Our results are expressed in terms of simple, intuitive measures of non-stationarity and correlation strength between different time series and hence explicitly depend on the key components of the learning
problem.

\begin{table*}[t]
  \caption{Summary of local, sequence-to-sequence, and hybrid models.}
  \label{tab:setup}
  \begin{center}
    \begin{small}
      \begin{tabular}{lcccr}
        \toprule
        \textsc{Learning Model} & \textsc{Training set} & \textsc{Hypothesis} & \textsc{Example} \\
        \midrule
        \textsc{UniVar. local}     & $\train_i=\{(\tsi Yi{t-1-p}{t-1}, Y_t(i)) : p \le  t \le T\}$& $h_i: \mathcal Y^p \to \mathcal Y$ & ARIMA \\
        \textsc{MultiVar. local}   & $\train=\{(\tsi Y\cdot{t-1-p}{t-1}, Y_t(\cdot)) : p \le t \le T\}$& $h: \mathcal Y^{m \times p} \to \mathcal Y^m$& VARMA \\
        \textsc{Seq-to-seq} & $\train=\{(\tsi Yi1{T-1}, Y_T(i)) : 1 \le i \le m\}$& $h: \mathcal Y^{T} \to \mathcal Y$& Neural nets\\
        \textsc{Hybrid}               & $\train=\{(\tsi Yi{t-p}{t-1}, Y_{t}(i)) : 1 \le i \le m, p \le t \le T\}$& $h: \mathcal Y^{p} \to \mathcal Y$& Neural nets \\
        \bottomrule
      \end{tabular}
    \end{small}
  \end{center}
\end{table*}

\begin{figure*}[t]
  \centering
  \begin{subfigure}{.45\textwidth}
    \def\svgwidth{\textwidth}
    \begin{center}
      \ifarxiv
      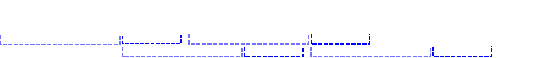
      \fi
    \end{center}
    \caption{The local model trains each $h_{\text{\scriptsize loc}, i}$ on time series $Y(i)$ split into multiple (partly overlapping) examples.}
  \end{subfigure}\hskip 1cm
  \begin{subfigure}{.45\textwidth}
    \begin{center}
    \def\svgwidth{\textwidth}
    \ifarxiv
    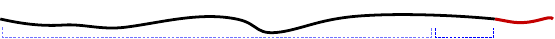
    \fi
    \vskip 1em
  \end{center}
    \caption{The sequence-to-sequence trains $h_{\text{\scriptsize s2s}}$ on $m$ time series split into (input, target) pairs.}
  \end{subfigure}
  \caption{Local and sequence-to-sequence splits of a one dimensional time series into training and test examples.}
  \label{fig:split-set}
\end{figure*}

We compare our generalization bounds to guarantees for local models
and identify regimes under which one methodology is superior to the other. Therefore, our theory may also serve as a quantitative guide for
a practitioner choosing the right modeling approach.

The rest of the paper is organized as follows: in Section~\ref{sec:setup},
we formally define sequence-to-sequence and local modeling.
In Section~\ref{sec:prelim}, we define the key tools that we require for
our analysis. Generalization bounds for sequence-to-sequence models
are given in Section~\ref{sec:dependent}. We compare sequence-to-sequence
and local models in Section~\ref{sec:comp}. Section~\ref{app:hybrid}
concludes this paper with a study of a setup that is a hybrid of the
local and sequence-to-sequence models.

\section{Sequence-to-sequence modeling}
\label{sec:setup}

We begin by providing a formal definition of sequence-to-sequence modeling.  The learner receives a multi-dimensional time series $\past \in \mathcal Y^{m \times T}$ which we view as $m$ time series of same length $T$. We denote by $Y_t(i)$ the value of the $i$-th time series at time $t$ and write $Y_a^b(i)$
to denote the sequence $(Y_a(i), Y_{a+1}(i), \ldots, Y_b(i))$.
Similarly, we let $Y_t(\cdot) = (Y_t(1), \ldots, Y_t(m))$ and
$Y_a^b(\cdot) = (Y_a(\cdot), \ldots, Y_b(\cdot))$. In particular, $\past \equiv Y_1^T(\cdot)$.
In addition, the sequence $Y_1^{T-1}(\cdot)$ is of a particular importance in our
analysis and we denote it by $\mathbf{Y}'$.

The goal of the learner is to predict $Y_{T+1}(\cdot)$.\footnote{We are often interested in long term forecasting, i.e. predicting $Y_{T+1}^{T+\ell}(\cdot)$ for $\ell \geq 1$. For simplicity, we only consider the case of $\ell=1$. However, all our results extend immediately to $\ell \ge 1$.}  We further assume that our input $\past$ is partitioned into a training set of $m$ examples
$\train = \{Z_1, \ldots, Z_m\}$, where each $Z_i = (Y_1^{T-1}(i), Y_T(i)) \in \mathcal Y^T$. The learner's objective is
to select a hypothesis
$h: \mathcal{Y}^{T} \to \mathcal{Y}$ from a given hypothesis set $\mathcal H$
that achieves a small generalization error:
\begin{equation*}
  \error(h \mid \past) = \frac{1}{m} \sum_{i=1}^m \expect_{\nextdist} \left[L(h(Y_1^T(i)), Y_{T+1}(i)) \mid \past \right],
\end{equation*}
where $L \colon \mathcal{Y} \times \mathcal{Y} \rightarrow [0,M]$
is a bounded\footnote{Most of the results in this paper
  can be straightforwardly extended to unbounded case assuming $Y$ is
  sub-Gaussian.} loss function and
$\nextdist$ is the distribution of $Y_{T+1}$ conditioned on the past $\past$.

In other words, the learner seeks a hypothesis $h$ that maps sequences of past $Y_1(i), \ldots, Y_T(i)$ values to sequences of future values $Y_{T+1}(i), \ldots, Y_{T+\ell}(i)$, justifying our choice of ``sequence-to-sequence'' terminology.\footnote{In practice, each $Z_i$ may start at a different, arbitrary time $t_i$, and may furthermore include some additional features $X_i$, i.e. $Z_i = (Y_{t_i}^{T-1}(i), X_i, Y_{T}(i))$. Our results can be extended to this case as well using an appropriate choice of the hypothesis set.} Incidentally, the machine translation problem studied by~\citet{NIPS2014_5346} under the same name represents a special case of our problem when sequences (sentences) are independent and data is stationary. In fact, LSTM-based approaches used in aforementioned translation problem are also common for time series forecasting~\citep{deepar,GoelMylnykBanerjee2017,YuZhengAnandkumarYue2017,LiYuShahabiLiu2017,laptev2017,zhu2017}. However, feed-forward NNs have also been successfully applied in this framework~\citep{romeu2013} and  in practice, our definition allows for any set of functions $\mathcal H$
that map input sequences to output sequences. For instance, we can train a feed-forward NN to map $Y_1^{T-1}(i)$ to $Y_T(i)$ and at inference time use
$Y_2^T(i)$ as input to obtain a forecast for $Y_T(i)$.\footnote{As another example, a runner-up in the Kaggle forecasting competition (\url{https://www.kaggle.com/c/web-traffic-time-series-forecasting}) used a combination of boosted decision trees and feed-forward networks, and as such employs the sequence-to-sequence approach.}

We contrast sequence-to-sequence modeling to \emph{local} modeling, which consists of splitting each time series $Y(i)$ into
a training set $\train_i = \{Z_{i,1}, \ldots, Z_{i,T}\}$, where
$Z_{i,t} = (Y_{t-p}^{t-1}(i), Y_t(i))$ for some $p \in \mathbb N$, then learning a separate  
hypothesis $h_i$ for each $\train_i$. 
Each $h_i$ models relations between observations that are close in time,  which is why we refer to this framework as local modeling.
As in sequence-to-sequence modeling, the goal of a local learner is to achieve a small generalization error for $h_{\text{loc}} = (h_1, \ldots, h_m)$, given by:
\begin{equation*}
  \error(h_{\text{loc}} \mid \past) = \frac{1}{m} \sum_{i=1}^m \expect_{\nextdist} \left[L(h_i(Y_{T-p}^T(i)), Y_{T+1}(i)) \mid \past \right].
\end{equation*}
In order to model correlations between different time series, it is also
common to split $\past$ into one single set of multivariate examples $\train = \{Z_{1}, \ldots, Z_T\}$, where
$Z_i = (Y_{t-p}^{t-1}(\cdot), Y_t(\cdot))$, and to learn a single hypothesis $h$ that maps $\mathcal{Y}^{m \times p} \to \mathcal{Y}^m$. 
As mentioned earlier, we consider this approach a variant of local modeling, since $h$ in this case
again models relations between observations that are close in time.

Finally, \emph{hybrid} or \emph{local sequence-to-sequence} models, which
interpolate between local and sequence-to-sequence approaches, have also been considered in the literature~\citep{zhu2017}.
In this setting, each local example is split across the temporal dimension into smaller examples of length $p$, which are then used to train a single sequence-to-sequence model $h$. We discuss bounds for this specific case in Section~\ref{app:hybrid}.

Our work focuses on the statistical properties of the sequence-to-sequence model.
We provide the first generalization bounds for sequence-to-sequence and hybrid models, and compare these to similar bounds for local models, allowing us to identify regimes in which one methodology is more likely to succeed. \ignore{Our theory can serve as a guide to a practioner choosing between different modeling approaches.}

Aside from learning guarantees, there are other important considerations that may lead a practitioner to choose one approach over others. For instance,
the local approach is trivially parallelizable; on the other hand, when additional features $X_i$ are available, sequence-to-sequence modeling provides an elegant solution to the \emph{cold start problem} in which at test time we are required to make predictions on time series for which no historical data is available.


\newcommand{\breakingcomma}{%
  \begingroup\lccode`~=`,
  \lowercase{\enalign*\expandafter\def\expandafter~\expandafter{~\penalty0 }}}

\section{Correlations and non-stationarity}
\label{sec:prelim}

In the standard supervised learning scenario, it is common to assume that
training and test data are drawn {\it i.i.d.}~from some unknown distribution.
However, this assumption does not hold for time series, where
observations at different times as well as across different series
may be correlated. Furthermore, the data-generating distribution
may also evolve over time.

These phenomena present a significant
challenge to providing guarantees in time series forecasting.
To quantify non-stationarity and correlations, we introduce the notions of
mixing coefficients and discrepancy, which are defined below.

The final ingredient we need to analyze  sequence-to-sequence learning is the Rademacher complexity $\radem_m(\mathcal F)$ of a family of functions $\mathcal F$ on a sample of size $m$, which has been previously used to characterize learning in the i.i.d.~setting~\citep{koltchinskii2002,mohri-book}.
In App.~\ref{app:radem}, we include a brief discussion of its properties.

\subsection{Expected mixing coefficients}

To measure the strength of dependency between time series,
we extend the notion of $\beta$-mixing coefficients~\citep{doukhan1994} to
\emph{expected $\beta$-mixing coefficients}, which are
a more appropriate measure of correlation in sequence-to-sequence modeling.

\begin{defn}[Expected $\hbeta$ coefficients]
  Let $i, j \in [m] \triangleq \{1, \ldots, m\}$. We define
  \begin{align*}
    \hbeta(i,j)=&~ \expect_{\tpast} \Big[\|P(Y_{T}(i) | \tpast ) P(Y_{T}(j) | \tpast ) -  P(Y_{T}(i), Y_{T}(j) | \tpast )\|_{TV}\Big],
  \end{align*}
  where $TV$ denotes the total variations norm. For a subset $I \subseteq [m]$, we define \[\hbeta(I) = \sup_{i, j \in C} \hbeta(i,j).\]
\end{defn}
The coefficient $\hbeta(i,j)$ captures how close $Y_{T+1}(i)$ and $Y_{T+1}(j)$ are to being independent, given $\past'$ (and averaged over all realizations of $\past'$). We further study these coefficients in Section~\ref{sec:dependent}, where we derive explicit upper bounds on expected $\hbeta$-mixing coefficients for various standard classes of stochastic processes, including spatio-temporal and hierarchical time series.

We also define the following related notion of $\bar\beta$-coefficients.
\begin{defn}[Unconditional $\bar\beta$ coefficients]
  Let $i, j \in [m] \triangleq \{1, \ldots, m\}$. We define
\begin{align*}
    \bar\beta(i,j)\! =& \|\Pr(\tsi Yi1T, \tsi Yj1T) - \Pr(\tsi Yi1T)\Pr(\tsi Yj1T)\|_{TV}\\
    \bar\beta'(i,j)\! =& \|\Pr(\tsi Yi1{T-1}, \tsi Yj1{T-1}) - \Pr(\tsi Yi1{T-1})\Pr(\tsi Yj1{T-1})\|_{TV}
\end{align*}
and as before, for a subset $I$ of $[m]$, write $\bar \beta(I) = \sup_{i,j \in I} \bar\beta(i,j)$ (and similarly for $\bar \beta'$).
\end{defn}

Note that $\hbeta$ coefficients measure the strength of dependence between time
series conditioned on the history observed so far, while $\bar \beta$ coefficients measure the (unconditional) strength of dependence between time series. The following result relates these two notions.

\begin{lemma}
  For $\bar \beta$ (and $\bar \beta'$ similarly), we have the following upper bound:
  \begin{align*}
    \bar \beta(i,j) \le& \hbeta(i,j) + \expect_{\past'}\Big[\text{Cov}\Big(\Pr(Y_T(i) \hiderel \mid \past'), \Pr(Y_T(j) \hiderel \mid \past')\Big)\Big]
  \end{align*}
\end{lemma}
The proof of this result (as well as all other proofs in this paper) is deferred to the supplementary material.

Finally, we require the notion of \emph{tangent collections}, within which time series are independent.
\begin{defn}[Tangent collection]
  Given a collection of time series $C=\{Y(1), \ldots, Y(c)\}$, we define the tangent collection $\widetilde C$ as $\{\widetilde Y(1), \ldots, \widetilde Y(c)\}$ such that $\widetilde Y(i)$ is drawn according to the marginal $\Pr(Y(i))$ and such that $\widetilde Y(i)$ and $\widetilde Y(i')$ are independent for $i\neq i'$.
\end{defn}

The notion of tangent collections, combined with mixing coefficients, allows us to reduce the analysis of correlated time series in $C$ to the analysis of independent time series in $\widetilde C$ (see Prop.~\ref{prop:yu} in the appendix). 

\subsection{Discrepancy}
Various notions of discrepancy have been previously used to measure
the non-stationarity of the underlying stochastic processes with respect
to the hypothesis set $\mathcal H$ and loss function
$L$ in the analysis of local models~\citep{kuznetsov2015,zimin2017}.
In this work, we introduce a notion of discrepancy specifically tailored to
sequence-to-sequence modeling scenario, taking into account both the hypothesis set and the loss function.
\begin{defn}[Discrepancy]
  \label{def:discrepancy}
  Let $\nextdist$ be the distribution of $Y_{T+1}$ conditioned on $\past$ and let $\dist$ be the distribution of $Y_T$ conditioned on $\mathbf{Y}'$. We define the discrepancy $\disc$ as
$   \Delta =\sup_{h \in \mathcal H} \left|\error(h \mid \past) - \error(h \mid \mathbf{Y}') \right|$ 
  where
$\error(h\mid \past') = \frac{1}{m} \sum_{i=1}^m \expect_{\dist} \left[L(h(Y_1^{T-1}(i)), Y_{T}(i)) \mid \past \right]$.
\end{defn}

The discrepancy forms a pseudo-metric on the space of probability distributions
and can be completed to a Wasserstein metric (by extending $\mathcal H$ to all
Lipschitz functions). This also immediately implies that the discrepancy
can be further upper bounded by the $l_1$-distance and
by relative entropy between conditional distributions of $Y_T$ and $Y_{T+1}$
(via Pinsker's inequality).
However, unlike these other divergences, the discrepancy takes into
account both the hypothesis set and the loss function, making it a finer
measure of non-stationarity.

However, the most important
property of the discrepancy is that it can be upper
bounded by the related notion of \emph{symmetric discrepancy}, which can be estimated from data.
\begin{defn}[Symmetric discrepancy]
  \label{def:sym-discrepancy}
 We define the symmetric discrepancy $\disc_s$ as
 \begin{align*}
   \disc_s =& \frac{1}{m} \sup_{h,h' \in \mathcal H} \Big|  \nlsum_{i=1}^m L(h(\ts Y{1}{T}(i)), h'(\ts Y{1}{T}(i))) -L(h(\ts Y{1}{T-1}(i)), h'(\ts Y{1}{T-1}(i)))\Big|.
 \end{align*}
\end{defn}

\begin{prop}
  \label{prop:disc_s}
  Let $\mathcal H$ be a hypothesis space and let $L$ be a bounded loss function which respects the triangle inequality. Let $h \in \mathcal H$ be any hypothesis. Then,
$    \disc \le \disc_s + \error(h \hiderel\mid \past) + \error(h \hiderel\mid \past').
$
\end{prop}
We do not require test labels to
evaluate $\disc_s$. Since $\disc_s$
only depends on the observed data, $\disc_s$ can be computed directly
from samples, making it a useful tool to assess
the non-stationarity of the learning problem.

Another useful property of $\disc_s$ is that, for certain classes of
stochastic processes, we can provide a direct analysis of this quantity.
\begin{prop}
  \label{prop:disc_bound}
  Let $I_1, \cdots, I_k$ be a partition of $\{1, \ldots, m\}$, $C_1, \ldots, C_k$ be the corresponding partition of $\past$ and
$C'_1, \ldots, C'_k$ be the corresponding partition of $\past'$. Write $c=\min_j |C_j|$, and define the \emph{expected discrepancy}
    \begin{align*}
    \disc_e =& \sup_{h,h' \in \mathcal H} \Big[\expect_Y[L(h(\ts Y1{T}),h'(\ts Y1{T}))] - \expect_Y[L(h(\ts Y1{T-1}),h'(\ts Y1{T-1}))]\Big].
  \end{align*}
  Then, writing $\radem$ the Rademacher complexity (see Appendix~\ref{app:radem}) we have with probability $1-\delta$,
  \begin{align*}
    \disc_s \le& \disc_e + {\max \Big(\max_j \radem_{|C_j|}(\widetilde C'_j), \max_j \radem_{|C_j|}(\widetilde C_j)\Big)} + \sqrt{\frac 1{2c}\log \frac{2k}{\delta - \sum_j (|I_j|-1)[\bar \beta(I_j)+ \bar \beta'(I_j)]}}.
  \end{align*}
\end{prop}

The expected discrepancy $\disc_e$  can be computed analytically for many
classes of stochastic processes. For example, for stationary processes,
we can show that it is negligible. Similarly, for covariance-stationary\footnote{Recall that a process $X_1, X_2, \ldots$
  is called stationary if for any $l,k,m$, the distributions of
  $(X_k, \ldots, X_{k+l})$ and $(X_{k+m}, \ldots, X_{k+m+l})$ are the same.
  Covariance stationarity is a weaker condition that requires
  that $\expect[X_k]$ be independent of $k$ and that $\expect[X_k X_m] = f(k-m)$
for some $f$.}
processes with linear hypothesis sets and the squared loss function,
the discrepancy is once again negligible.
These examples justify our use of the discrepancy as a natural measure of
non-stationarity. In particular, the covariance-stationary example highlights that
the discrepancy takes into account not only the distribution of the stochastic
processes but also $\mathcal{H}$ and $L$.

\begin{prop}
  \label{prop:disc_stat}
  If $Y(i)$ is stationary for all $1 \le i \le m$, and $\mathcal H$ is a hypothesis space such that $h \in \mathcal H : \mathcal Y^{T-1} \to \mathcal Y$ (i.e. the hypotheses only consider the last $T-1$ values of $Y$), then $\disc_e =0$.
\end{prop}
\begin{prop}
  \label{prop:disc_cov_stat}
  If $\past$ is covariance stationary for all $1 \le i \le m$, $L$ is the squared loss, and $\mathcal H$ is a linear hypothesis space $\{x \to w \cdot x \mid \|w\| \in \mathbb R^p \le \Lambda\}$, $\disc_e=0$.
\end{prop}
Another insightful example is the case
when $\mathcal H = \{h\}$: then, $\disc = 0$ even if $\past$
is non-stationary, which illustrates that learning is trivial for trivial hypothesis sets, even in non-stationary settings.

The final example that we consider in this section is the case of non-stationary periodic time series. Remarkably, we show that the discrepancy is still negligible in this case provided that we observe all periods with equal probability.

\begin{prop}
  \label{prop:disc_period}
  If the $Y(i)$ are periodic with period $p$ and the observed starting time of each $Y(i)$ is distributed uniformly at random in $[p]$, then $\disc_e=0$.
\end{prop}


\section{Generalization bounds}
\label{sec:dependent}
We now present our generalization bounds for time series prediction with sequence-to-sequence models. We write $\mathcal F = \{L \circ h: ~h \in \mathcal H\}$, where $f = L \circ h$ is the loss of hypothesis $h$ given by $f(h, Z_i) = L(h(\ts Y1{T-1}(i)),Y_{T})$. To obtain bounds on the generalization error $\error(h \mid \past)$,
we study the gap between $\error(h \mid \past)$ and the empirical error
$\widehat{\mathcal{L}}(h)$ of a hypothesis $h$, where
\[\widehat \error(h) = \frac{1}{m} \sum_{i=1}^m f\big(h, Z_i\big).\]

That is, we aim to give a high probability bound on
the supremum of the empirical process
$\Phi(\past) = \sup_h [ \error(h \mid \past) - \widehat \error(h)]$.
We take the following high-level approach: we first partition the training set \train into $k$ collections $C_1, \ldots, C_k$ such that within each collection, correlations between different time series are as weak as possible. We then analyze each collection $C_j$ by comparing the generalization error of sequence-to-sequence learning on $C_j$ to the sequence-to-sequence generalization error on the \emph{tangent} collection $\widetilde C_j$.

\ignore{
  For ease of notation, we write $\radem(\widetilde C)$ the Rademacher complexity of $\mathcal F$ on data drawn from the same distribution as $\widetilde C$; $\disc(\widetilde C)$ similarly, and

Define
\[\hphi(C_j) = \sup_{h \in \mathcal H} \Big[\error(h, \nextdist) - \frac 1{|C_j|} \sum_{Z \in C_j} f(h, Z)\Big].\]}
\begin{theorem}
  \label{thm:dep}
  Let $C_1, \ldots, C_k$ form a partition of the training input \train and
  let $I_j$ denote the set of indices of time series that belong to $C_j$. Assume that the loss function $L$ is bounded by 1. Then, we have for any $\delta > \sum_j (|I_j| - 1)\beta(I_j)$, with probability $1-\delta$, 
  \begin{align*}
  \Phi&(\past) \leq \max_j \left[\widehat \radem_{\widetilde C_j}(\mathcal F)\right] + \disc + \frac 1 {\sqrt {2\min_j |I_j|}} \sqrt{\log\left(\frac{k}{\delta - \sum_j (|I_j|-1)\hbeta(I_j)}\right)}.
  \end{align*}
\end{theorem}
 
%
Theorem~\ref{thm:dep} illustrates the trade-offs that are involved
in sequence-to-sequence learning for time series forecasting. As $\sum_j (|I_j| - 1)\beta(I_j)$ is a function of $m$, we expect it to decrease as $m$ grows (\emph{i.e.} more time series we have), allowing for smaller $\delta$ as $m$ increases.

Assuming that the $C_j$ are of the same size,
if $\mathcal H$ is a collection of neural networks
of bounded depth and width then
$\radem_{\widetilde C_j}(\mathcal F) = \mathcal O\Big(\sqrt{{k T}/{m}}\Big)$
(see Appendix~\ref{app:radem}). Therefore,
\[  {\error(h \mid \past) \leq \widehat \error(h) +  \disc }  + \mathcal O\Big(\sqrt{\tfrac{k T}{m}}\Big)\]
with high probability uniformly over $h \in \mathcal{H}$,
provided that $\tfrac{m}{k} \sum_{j=1}^k \hbeta(I_j) = o(1)$.
This shows that extremely high-dimensional ($m \gg  1$) time series
are beneficial for sequence-to-sequence models, whereas series with
a long histories $T \gg m$ will generally not benefit from sequence-to-sequence learning.
Note also that correlations in data reduce the effective sample size from
$m$ to $m/k$.

Furtermore, Theorem~\ref{thm:dep} indicates that balancing the complexity of the model (e.g. depth and width of a neural net) with the fit it provides to the data
is critical for controlling both the discrepancy and Rademacher complexity
terms.
We further illustrate this bound with several examples below.
\ignore{
This bound improves as size of the smallest collection size increases: hence, we ideally can construct $C_1, \ldots, C_k$ such that $|C_j| \approx m/k$ for all $j$ without paying a penalty of high values for $\hbeta(C_j)$.

\begin{rmk}
 In all results, we use the Rademacher complexity $\radem_m$; however, the same bounds can be derived using the \emph{empirical} Rademacher complexity via Eq.~\eqref{eq:kp-emp} from Appendix~\ref{app:kp}.
\end{rmk}
Note that when time series are independent, we can set $C=\train$ and $k=1$: the term in $\delta$ grows as $\mathcal O\Big(\sqrt{\tfrac 1m \log\tfrac 1\delta}\Big)$, as derived in section~\ref{sec:independent}.
\begin{rmk}
Conversely, if all time series are highly dependent, we must set $k=m$ and each $C_j$ to the singleton $\big\{\tsi Zj{t_j}{t_j+T}\big\}$. The bound grows as $\mathcal O(\sqrt{\log \tfrac m\delta})$; we show below that in this case, the traditional setup is preferable.
\end{rmk}}

\subsection{Independent time series}
\label{sec:independent}
We begin by considering the  case where all dimensions of $\past$ are independent. Although this may seem a restrictive assumption, it arises in a variety of applications: in neuroscience, different dimensions 
may represent brain scans of different patients; in reinforcement
learning, they may correspond to different trajectories of a robotic arm.
\begin{theorem}
  \label{prop:gen-bound}
  Let $\mathcal H$ be a given hypothesis space with associated function family $\mathcal F$ corresponding to a loss function $L$ bounded by 1. Suppose that all dimensions of $\past$ are independent and let $I_1 = [m]$; then $\beta(I_1) = 0$ and so for any $\delta >0$, with probability at least $1-\delta$ and for any $h \in \mathcal H$:
  \begin{equation*}
    \error(h | \past) \le \widehat \error(h) + 2\mathfrak R_m(\mathcal F) + \disc + \sqrt{\frac {\log(1/\delta)}{m}}.
  \end{equation*}
\end{theorem}
Theorem~\ref{prop:gen-bound} shows that when time series are independent,
learning is not affected by
correlations in the samples and can only be obstructed by the non-stationarity
of the problem, captured via $\disc$.

Note that when examples are drawn \textit{i.i.d.}, we have $\Delta=0$ in Theorem~\ref{prop:gen-bound}: we recover the standard standard generalization results for regression problems.

\subsection{Correlated time series}
We now consider several concrete examples of high-dimensional
time series in which different dimensions may be correlated. This setting
is common in a variety of applications including stock market indicators,
traffic conditions, climate observations at different locations, and
energy demand.
\ignore{In order to relate the less obvious bound from Theorem~\ref{thm:dep} to concrete examples, we now describe the guarantees obtained when assuming different schemes to generate dependent families of Auto-Regressive time series.}

Suppose that each $Y(i)$ is generated by the auto-regressive
(AR) processes with correlated noise
\begin{equation}
y_{t+1}(i) = \Theta_i( y^{t}_{0}(i) ) + \varepsilon_{t+1}(i)\label{eq:gaussian}
\end{equation}
where the $w_i \in \mathbb R^p$ are unknown parameters and the noise vectors $\epsilon_t \in \mathbb R^m$ are drawn from a Gaussian distribution $\mathcal N(0, \Sigma)$ where, crucially, $\Sigma$ is not diagonal. The following lemma is key to our analysis. \ignore{We will also require the following lemma, proven in App.~\ref{app:hierarchy}.}
\begin{lemma}
  \label{lemma:gaussian}
    Two AR processes $Y(i),Y(j)$ generated by~\eqref{eq:gaussian} such that $\sigma = \text{Cov}(Y(i),Y(j)) \le \sigma_0 < 1$ verify $\hbeta(i,j) = \max \left(\frac 3{2(1-\sigma_0^2)}, \frac 1 {1-2\sigma_0}\right)\sigma = \mathcal O(\sigma)$.
\end{lemma}
 
\paragraph{Hierarchical time series.}
As our first example, we consider the case of hierarchical time series
that arises in many real-world applications~\citep{RePEc:msh:ebswps:2015-15,pmlr-v70-taieb17a}. Consider the problem
of energy demand forecasting: frequently, one observes a sequence of energy demands at a variety of levels: single household, local neighborhood,
city, region and country. This imposes a natural hierarchical structure
on these time series.

\ifarxiv
\begin{wrapfigure}{r}{6.5cm}
  \centering
  \begin{tikzpicture}[->,>=stealth',level/.style={sibling distance = 2.9cm/#1,
      level distance = .8cm}]
    \node [interm] {Global}
    child{ node [interm] {France}
      child{ node [interm] {Paris}
        child{ node [leaf] {{\color{red}{\tiny 1st ar.}}}}
        child{ node [leaf] {{\color{blue}{\tiny 2nd ar.}}}}
      }
      child{ node [interm] {...\vphantom{f}}
      }
    }
    child{ node [interm] {USA}
      child{ node [interm] {New York}
        child{ node [leaf] {{\color{red}{\tiny 5th Ave.}}}
        }
        child{ node [leaf] {{\color{blue}{\tiny 8th Ave.}}}
        }
      }
      child{ node [interm] {...\vphantom{f}}
      }
    };
  \end{tikzpicture}
  \caption{Hierarchical dependence. Collections for $d=1$ are given by $C_1$ which contains the red leaves and $C_2$ which contains the blue right-side leaves.}
\end{wrapfigure}
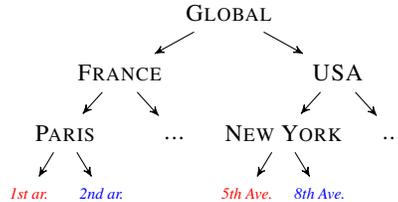
\fi
Formally, we consider the following hierarchical scenario: a binary tree of total depth $D$, where time series are generated at each of the leaves. At each leaf,  $Y(i)$ is given by the AR process~\eqref{eq:gaussian} where we impose $\Sigma_{ij} = (\tfrac 1m)^{d(i,j)}$ given $d(i,j)$ the length of the shortest path from either leaf to the closest common ancestor between $i$ and $j$. Hence, as $d(i,j)$ increases, $Y(i)$ and $Y(j)$ grow more independent. 

For the bound of Theorem~\ref{thm:dep} to be non-trivial, we require a partition $C_1, \ldots, C_k$ of \train such that within a given $C_j$ the time series are close to being independent. One such construction is the following: fix a depth $d \le D$ and construct $C_1, \ldots, C_{2^d}$ such that each $C_i$ contains exactly one time series from each sub-tree of depth $D-d$; hence, $|C_i| = 2^{D-d}$. Lemma~\ref{lemma:gaussian} shows that for each $C_i$, we have $\beta(C_i) = \mathcal O(m^{d-D})$. For example, setting $d=\tfrac D2 = \tfrac {\log m}2$, it follows that for any $\delta > 0$, with probability $1-\delta$,
\begin{align*}
  \error(h | \past) \le& \widehat \error(h) +  \max_j \left[\radem_{\widetilde C_j}(\mathcal F)\right] + \disc +\frac 1 {\sqrt 2 \sqrt[4]{m}} \sqrt{\log\left( \frac{\sqrt m}{\delta - \frac{m}{\mathcal O (\sqrt{m^{\log m}})}}\right)}.
\end{align*}
Furthermore, suppose the model is a linear AR process given by
$y_{t+1}(i) = w_i \cdot( y^{t}_{t-p}(i) ) + \varepsilon_{t+1}(i)$.
Then, the underlying stochastic process is weakly stationary and
by Prop.~\ref{prop:disc_stat} our bound reduces to:
$  \error(h | \past) \le \widehat \error(h) +  \max_j \left[\radem_{\widetilde C_j}(\mathcal F)\right] + \mathcal{O}\Big(\frac{\sqrt{\log m}}{m^{1/4}}\Big)$.
By Proposition~\ref{prop:disc_period}, similar results holds when $\Theta_i$ is periodic.

\paragraph{Spatio-temporal processes.}
Another common task is spatio-temporal forecasting,
in which historical observation
are available at different locations. These observations may represent temperature at different locations, as in the case of climate modeling~\citep{AAAI125158,DBLP:conf/aaai/GhafarianzadehM13}, or car traffic at different locations~\citep{LiYuShahabiLiu2017}.

It is natural to expect correlations between time series to
decay as the geographical distance between them increases.
As a simplified example, consider that the sphere $\mathbb S^3$ is subdivided according to a geodesic grid and a time series is drawn from the center of each patch according to~\eqref{eq:gaussian}, also with $\Sigma_{ij} = m^{-d(i,j)}$ but this time with $d(i,j)$ equal to the (geodesic) distance between the center of two cell centers.
We choose subsets $C_i$ with the goal of minimizing the strength of
dependencies between time series within each subsets.
Assuming we  divide the sphere into $\sqrt m$ collections size  approximately $c=\sqrt m$ such that the minimal distance between two points
in a set is $d_0$, we obtain
\begin{align*}
  \error(h \mid \past) \le& \widehat \error(h) +  \max_j \left[\radem_{\widetilde C_j}(\mathcal F)  \right] + \disc +\frac 1 {\sqrt 2 \sqrt[4]m}\sqrt{\log\left(\frac{\sqrt m}{\delta - \mathcal O(m^{1-d_0})}\right)}.
\end{align*}

As in the case of hierarchical time series, Proposition~\ref{prop:disc_stat} or
Proposition~\ref{prop:disc_period} can be used to remove the dependence on $\disc$ for
certain families of stochastic processes.

\ignore{
\paragraph{General graphs.} The reasoning for the hierarchical and spatio-temporal time series can also be extended to arbitrary graphical models. In particular, Bayes networks lend themselves well to such an analysis, with the added benefit that the structure of a Bayes net can be inferred from data.}

\section{Comparison to local models}
\label{sec:comp}
This section provides comparison of learning guarantees for sequence-to-sequence
models with those of local models. In particular,
we will compare our bounds on the generalization gap $\Phi(\past)$
for sequence-to-sequence models and local models, where the gap is given by
\begin{align}
  \nhphi(\past) = \sup_{(h_1, \ldots, h_m) \in H^m}
  \Big[ \error(h_{\text{loc}} \mid \past) - \widehat \error(h_{\text{loc}})\Big]
\end{align}
where $\widehat \error(h_{\text{loc}} )$ is the average empirical
error of $h_i$ on the sample $\train_i$, defined as
$  \widehat \error(h_{\text{loc}}) = \frac{1}{mT} \nlsum_{i=1}^m \nlsum_{t=1}^T f(h_i, Z_{t, i})$
where $f(h_i, Z_{t, i}) = L(h_i(Y_{t-p}^{t-1}(i)), Y_t(i))$.

To give a high probability bound for this setting, we take advantage of existing results for the single local model $h_i$
\citep{kuznetsov2015}. These results are given in terms of a slightly different
notion of discrepancy $\Delta$, defined by
\begin{align*}
  \Delta(\train_i) =& \sup_{h \in \mathcal H} \Bigg[
    \expect\Big[ L(h(Y_{t-p+1}^T), Y_{T+1}) \mid Y_1^T\Big] - \frac{1}{T} \sum_{t=1}^T \expect\Big[ L(h(Y_{t-p}^{t-1}), Y_{t}) \mid Y_1^{t-1}\Big] \Bigg].
\end{align*}
Another required ingredient to state these results is
the expected sequential covering number
$\expect_{v \sim T(\mathbb P)} [\mathcal N_1(\alpha, \mathcal F, v)]$~\citep{kuznetsov2015}.
For many hypothesis sets, the log of the sequential covering number
admits upper bounds similar to those presented earlier for the Rademacher
complexity. We provide some examples below and
refer the interested reader to \citep{rakhlin2015} for a details.

\ignore{To bound $\nhphi(\train)$, we refer to \citet[Corollary 2]{kuznetsov2015}: 
\begin{equation*}
  \expect[f(h, \tau(Z)) \mid Z] - \frac 1T\sum_{t=1}^T f(h, \ts Z1t) \le \Delta(Z) + 2\alpha + \frac 1 {\sqrt T} \sqrt{2 \log \frac{\expect_{v \sim T(Z)} [\mathcal N_1(\alpha, \mathcal F, v)]} \delta}
\end{equation*}
where $\alpha >0$, $\Delta(Z)$ indicates the discrepancy as defined for the local approach, and $\expect_{v \sim T(\mathbb P)} [\mathcal N_1(\alpha, \mathcal F, v)]$, defined in detail in~\citep{kuznetsov2015}, corresponds to the expected sequential covering number of the data-dependent binary tree generated by the $\{Z_{it}\}_{t \le T}$ . By union bound, we then obtain the following result}

\begin{theorem}
\label{th:local}
For $\delta > 0$ and $\alpha > 0$,  with probability at least $1-\delta$, for any $(h_1, \ldots, h_m)$,
  and any $\alpha > 0$, 
\begin{align*}
  \nhphi(\past) \le& \frac 1m \sum_{i=1}^m \Delta(\train_i) + 2 \alpha
                     +
  \sqrt{\frac 2T \log \frac{m\, \max_i \expect_{v \sim T(\train_i)} [ \mathcal N_1(\alpha, \mathcal F, v)]}{\delta}}\label{eq:nh}.
\end{align*}
\end{theorem}

Choosing $\alpha = 1/\sqrt{T}$, we can show that, for standard local
models such as the linear hypothesis space $\{x \to w^\top x, w \in \mathbb R^p, \|w\|_2 \le \Lambda\}$, we have
\begin{align*}
\sqrt{\frac 1T \log \frac{2m\,\expect_{v \sim T(Z)} [\mathcal N_1(\alpha, \mathcal F, v)]}{\delta}} = \mathcal O\Big(\sqrt{\frac{\log m}{T}}\Big).
\end{align*}
In this case, it follows that
$\nhphi(\past) \le \frac 1m \sum_{i=1}^m \Delta(\train_i) +
  \mathcal O\Big(\sqrt{\frac{\log m}{T}}\Big)$.
where the last term in this bound should be compared
with corresponding (non-discrepancy) terms in the bound
of Theorem.~\ref{thm:dep}, which, as discussed above, scales
as $\mathcal{O}(\sqrt{T/m})$ for a variety of different hypothesis sets.

Hence, when we have access to relatively few time series compared to their length ($m \ll T$), learning to predict each time series as its own independent problem will with high probability lead to a better generalization bound. On the other hand, in extremely
high-dimensional settings when we have significantly more time series than time steps ($m \gg T$), sequence-to-sequence learning will (with high probability) provide superior performance. We also expect the performance of sequence-to-sequence models to deteriorate as the correlation between time series increases.

A direct comparison of bounds in Theorem~\ref{thm:dep} and Theorem~\ref{th:local} is complicated by the fact that discrepancies that appear in these results are different. In fact, it is possible to
design examples where $\tfrac{1}{m} \sum_{i=1}^m \Delta(\train_i)$ is
constant and $\disc$ is negligible, and vice-versa.

Consider a tent function $g_b$ such that $g_b(s) = 2bs / T$ for  $s \in[0, T/2]$ and $g_b(s) = -2bs / T + 2b$ for $s \in [T/2, T]$. Let $f_b$ be its periodic
extension to the real line, and define $\mathcal S = \{f_b \colon b \in [0, 1]\}$. Suppose that we sample uniformly 
$b \in [0, 1]$ and $s \in \{0, T/2\}$ $m$ times, and observe time series
$f_{b_i}(s_i), \ldots, f_{b_i}(s_i +T)$. Then, as we have shown
in Proposition~\ref{prop:disc_period},
$\disc$ is negligible for sequence-to-sequence models. However,
unless the model class is trivial, it can be shown that
$\Delta(\train_i)$ is bounded away from zero for all $i$.

Conversely, suppose we sample uniformly $b \in [0, 1]$ $m$ times and observe time series
$f_{b_i}(0), \ldots, f_{b_i}(T/2+1)$. Consider a set of local models
that learn an offset from the previous point $\{ h \colon x \mapsto x + c, c \in [0, 1]\}$. It can be shown that in this case $\Delta(\train_i) = 0$, whereas $\disc$ is bounded away from zero for any non-trivial class of sequence-to-sequence models.

From a practical perspective, we can simply use $\disc_s$
and empirical estimates of $\Delta(\train_i)$ to decide whether to choose sequence-to-sequence or local models.

We conclude this section with an observation that similar results to Theorem~\ref{thm:dep} can
be proved for multivariate local models with the only difference
that the sample complexity of the problem scales as $O(\sqrt{m/T})$, and hence
these models are even more prone to the curse of dimensionality.

\section{Hybrid models}
\label{app:hybrid}

In this section, we discuss models that interpolate between
local and sequence-to-sequence models. This hybrid approach trains
a single model $h$ on the union of \emph{local} training sets $\train_1, \ldots, \train_m$
 used to train $m$ models in the local approach. The bounds that we state here require the following extension of the discrepancy to $\disc_t$, defined as
\begin{align*}
\disc_t =& \frac{1}{m} \sup_{h \in \mathcal H}\Big| \sum_{i =1}^m \expect_{\nextdist}[L(h( Y_{t-p-1}^{t-1}(i)), Y_{t}(i)) | Y_1^{t-1} ]
  - \expect_{\dist}[L(h( Y_{T-p}^{T}(i)), Y_{T+1}(i))| \past ]\Big| 
  \end{align*}
Many of the properties that were discussed for the discrepancy $\disc$ carry over to $\disc_t$ as well. The empirical error in this case is the same as for
the local models:
\begin{equation*}
  \widehat \error(h ) = \frac{1}{mT} \sum_{i=1}^m \sum_{t=1}^T f(h, Z_{t, i}).
\end{equation*}
Observe that one straightforward way to obtain a bound for hybrid models
is to apply Theorem~\ref{th:local} with $(h, \ldots, h) \in \mathcal H^m$.
Alternatively, we can apply Theorem~\ref{thm:dep} at every time
point $t=1, \ldots, T$.

Combining these results via union bound leads
to the following learning guarantee for hybrid models.
\begin{theorem}
  \label{thm:hybrid}
  \ignore{Let $\mathcal H$ be a hypothesis space, and $h \in \mathcal H$. }Let $C_1, \ldots, C_k$ form a partition of the training input \train and
  let $I_j$ denote the set of indices of time series that belong to $C_j$. Assume that the loss function $L$ is bounded by 1. Then, for any $\delta > 0$, with probability $1-\delta$, for any $h \in \mathcal{H}$ and any $\alpha > 0$
  \begin{equation*}
    {\error(h \mid \past) \leq \widehat \error(h) +  \min(B_1, B_2)},
  \end{equation*}
  where
  \begin{align*}
  B_1 = &\frac 1T \nlsum_{t=1}^T \disc_t + \max_j \widehat \radem_{\widetilde C_j}(\mathcal F) + \frac 1 {\sqrt {2\min_j |I_j|}} \sqrt{\log\left(\frac{2Tk}{\delta - 2\sum_j (|I_j|-1)\hbeta(I_j)}\right)} \\
  B_2 = &\frac 1m \nlsum_{i=1}^m \Delta(\train_i) + 2 \alpha +
  \sqrt{\frac 2T \log \frac{2 m\,  \max_i \expect_{v \sim T(\train_i)} [ \mathcal N_1(\alpha, \mathcal F, v)]}{\delta}}.
  \end{align*}

\end{theorem}

Using the same arguments for the complexity terms as in the
case of sequence-to-sequence and local models,
this result shows that hybrid models are successful with high probability when
$m \gg T$ or correlations between time series are strong, as well as when
$T \gg m$.

Potential costs for this model are hidden in the new
discrepancy term $\frac 1T \sum_{t=1}^T \disc_t$. This
term leads to different bounds depending on the particular
non-stationarity in the given problem. As before this trade-off
can be accessed empirically using the data-dependent version of
discrepancy.

Note that the above bound does not imply
that hybrid models are superior to local models:
using $m$ hypotheses $h_1, \ldots, h_m$ can help us
achieve a better trade-off between $\widehat \error(h)$ and $B_2$, and
vice versa.


\section{Conclusion}


\ignore{
Sequence-to-sequence learning has recently become relevant as large amounts of time series data are made available for analysis, and several state-of-the-art results in time series prediction use
neural network models which leverage this framework.}

We formally introduce sequence-to-sequence learning for time series, a framework in which a model learns to map past sequences of length $T$ to their next values. We provide the first generalization bounds for sequence-to-sequence modeling.
Our results are stated in terms of new notions of discrepancy
and expected mixing coefficients. We study these new notions
for several different families of stochastic processes
including stationary, weakly stationary, periodic,
hierarchical and spatio-temporal time series.

Furthermore, we show that our discrepancy can be computed from data, making
it a useful tool for practitioners to empirically assess the non-stationarity of their
problem. In particular, the discrepancy can be used to determine
whether the sequence-to-sequence methodology is
likely to succeed based on the inherent non-stationarity of the problem.

Furthermore, compared to the local framework for time series forecasting, in which independent models for each one-dimensional time series are learned, our analysis shows that the sample complexity of sequence-to-sequence models scales as $\mathcal O(\sqrt{T/m})$, providing superior guarantees when the number $m$ of time series is significantly greater than the length $T$ of each series, provided that different series are weakly correlated.

Conversely, we show that the sample complexity of local models scales as $\mathcal O(\sqrt{\log(m)/T})$, and should be preferred when $m \ll T$ or when time series
are strongly correlated. We also study hybrid models for which learning guarantees are favorable both when $m \gg T$ and $T \gg m$, but which have a more complex trade-off in terms of discrepancy.

As a final note, the analysis we have carried through is easily extended to show similar results for the sequence-to-sequence scenario when the test data includes \emph{new} series not observed during training, as is often the case in a variety of applications.

\bibliographystyle{plainnat}
{\bibliography{time_series.bib}}


\clearpage
\appendix
\appendix
\section{Rademacher complexity}
\label{app:radem}
\begin{defn}[Rademacher complexity]
  Given a family of functions $\mathcal F$ and a training set $\train=\{Z_1, \ldots, Z_m\}$, \emph{the Rademacher complexity of $\mathcal F$ conditioned on $\past'$} is given by
  \begin{equation*}
    \widehat{\mathfrak R}_\train(\mathcal F) = \expect_{\train, \sigma} \left[\max_{f \in \mathcal F} \frac 1m \nlsum_{i=1}^m \sigma_i f(Z_i) \Big| \past' \right]
  \end{equation*}
  where $\sigma_1, \ldots, \sigma_m$ are i.i.d.~random variables uniform on
  $\{-1,+1\}$. \emph{The Rademacher complexity} of $\mathcal F$ for sample size $m$ is given by
  \begin{equation*}
    \mathfrak R_m(\mathcal F) = \expect_{\past'} \left[\widehat{\mathfrak R}_\train(\mathcal F)\right].
  \end{equation*}
\end{defn}

\ignore{$\mathcal R_m(\mathcal H)$ is a general measure of $\mathcal H$'s ability to model the data, and is fundamental in bounding the gap between the empirical and expected error~\citep{koltchinskii2002,mohri-book} (see Theorem~\ref{thm:bounds} in Appendix~\ref{app:kp}); it typically scales as $\mathcal O(1/\sqrt m)$. }

The Rademacher complexity has been studied for a variety of function classes. For instance,
for the linear hypothesis space $\mathcal{H} = \{x \rightarrow w^\top x, \|w\|_2 \le \Lambda\}$,
$\widehat \radem_\train$ can be upper bounded by $\widehat \radem_\train(\mathcal H) \le \frac{\Lambda}{\sqrt{m}} \max_i \|Z_i\|_2 $.
As another example, the hypothesis class of ReLu feed-forward neural networks with $d$ layers and weight matrices $W_k$ such that $\prod_{k=1}^d \|W\|_F \le \gamma$ verifies $\widehat \radem_\train(\mathcal H) \le \frac {2^{d-1/2} \gamma} {\sqrt m}\max_i \|Z_i\|_2$~\citep{neyshabur15}.

\ignore{
We recall the following useful bounds:
\begin{itemize}
  \item If $f$ is $\lambda$-Lipschitz, $\radem_m(f \circ \mathcal H) \le \lambda \radem_m(\mathcal H)$. In particular, if our loss $L$ is $\lambda$-Lipschitz, $\radem_m(\mathcal F) \le \lambda \radem_m(\mathcal H)$.
  \item The linear hypothesis space $\{x \rightarrow w^\top x, \|w\|_2 \le \Lambda\}$ verifies $\widehat \radem_S(\mathcal H) \le \frac{\Lambda}{\sqrt{m}} \max_i \|x_i\|_2 $.
\end{itemize}
As an example, the hypothesis class of ReLu feed-forward neural networks with $d$ layers and weight matrices $W_k$ such that $\prod_{k=1}^d \|W\|_F \le \gamma$ verifies $\widehat \radem_S(\mathcal H) \le \frac {2^{d-1/2} \gamma} {\sqrt m}\max_i \|x_i\|_2$; we refer the interested reader to~\citep{neyshabur15} for a detailed analysis.
}

\section{Discrepancy analysis}
\label{app:sym-bound}

\begin{repproposition}{prop:disc_s}
Let $\mathcal H$ be a hypothesis space and let $L$ be a bounded loss function which respects the triangle inequality. Let $h' \in \mathcal H$. Then,
\begin{equation*}
  \disc \le \disc_s + \error(h \hiderel\mid \past) + \error(h \hiderel\mid \tpast)
\end{equation*}
\end{repproposition}
\begin{proof}
  Let $h, h' \in \mathcal H$. For ease of notation, we write
  \begin{align*}
    \disc_s(h,h', \tpast) =& \frac 1m \sum_iL(h(\tsi Yi1T), h'(\tsi Yi1T)) - \frac 1m \sum_i L(h(\tsi Yi1{T-1}), h'(\tsi Yi1{T-1})). 
  \end{align*}
  Applying the triangle inequality to $L$,
  \begin{align*}
      \error(h \hiderel\mid \past) =& \frac 1m \sum_i \expect[L(h(\tsi Yi1T), Y_{T+1}(i)) \hiderel\mid \past]\\
      \le& \frac 1m \sum_i L(h(\tsi Yi1T), h'(\tsi Yi1T)) + \frac 1m \sum_i\expect[L(h'(\tsi Yi1T), Y_{T+1}(i)) \hiderel\mid \past]\\
      =& \frac 1m \sum_i L(h(\tsi Yi1T), h'(\tsi Yi1T)) + \error(h' \hiderel\mid \past).
  \end{align*}\vskip -1em
  Then, by  definition of $\disc_s(h,h', \tpast)$, we have
  \begin{align*}
      \error(h \hiderel\mid \past) \le& \frac 1m \sum_iL(h(\tsi Yi1T), h'(\tsi Yi1T)) - \frac 1m \sum_i L(h(\tsi Yi1{T-1}), h'(\tsi Yi1{T-1}))\\ &+ \frac 1m \sum_i L(h(\tsi Yi1{T-1}), h'(\tsi Yi1{T-1}))  +   \error(h' \hiderel\mid \past)\\
      \le& \disc_s(h,h', \tpast) +   \error(h' \hiderel\mid \past) + \frac 1m \sum_i L(h(\tsi Yi1{T-1}), h'(\tsi Yi1{T-1})).
  \end{align*}\vskip -1em
  By an application of the triangle inequality to $L$,
  \begin{align*}
      \error(h, \nextdist) \le& \disc_s(h,h', \tpast)  + \error(h' \hiderel\mid \past) +\frac 1m \sum_i \expect[L(h(\tsi Yi1{T-1}), Y_T(i)) \hiderel\mid \tpast] \\&+ \frac 1m \sum_i \expect[L(h'(\tsi Yi1{T-1}), Y_T(i)) \hiderel\mid \tpast]\\
      =& \disc_s(h,h', \tpast) + \error(h' \hiderel\mid \past) + \error(h \hiderel\mid \tpast) + \error(h' \hiderel\mid \tpast).
  \end{align*}
  Finally, we obtain
  \begin{align*}
    \error(h \hiderel\mid \past) - \error(h \hiderel\mid \tpast) \le& \disc_s(h,h', \tpast) + \error(h' \hiderel\mid \past) + \error(h' \hiderel\mid \tpast) 
  \end{align*}
  and the result announced in the theorem follows by taking the supremum over $\mathcal H$ on both sides.
\end{proof}

\begin{repproposition}{prop:disc_bound}
 Let $I_1, \cdots, I_k$ be a partition of $\{1, \ldots, m\}$, and $C_1, \ldots, C_k$ be the corresponding partition of $\past$. Write $c=\min_j |C_j|$. Then we have with probability $1-\delta$,
  \begin{align*}
    \disc_s \le& \disc_e + {\max \Big(\max_j \radem_{|C_j|}(\widetilde C'_j), \max_j \radem_{|C_j|}(\widetilde I_j)\Big)} + \sqrt{\frac 1{2c}\log \frac{2k}{\delta - \sum_j (|I_j|-1)[\bar \beta(I_j)+ \bar \beta'(I_j)]}}.
  \end{align*}
\end{repproposition}
\begin{proof} By definition of $\disc_s$, 
  \begin{align*}
      \disc_s =& \sup_{h,h' \in \mathcal H} \frac 1m \sum_{i=1}^m \Big[L(h(\tsi Yi1{T}),h'(\tsi Yi1{T})) - L(h(\tsi Yi1{T-1}),h'(\tsi Yi1{T-1}))\Big]\\
              \le& \sup_{h,h' \in \mathcal H}  \Big[\frac 1m \sum_{i=1}^m L(h(\tsi Yi1{T}),h'(\tsi Yi1{T})) - \expect_Y[L(h(\ts Y1{T}),h'(\ts Y1{T}))]\Big] \\
               &+ \sup_{h,h' \in \mathcal H} \Big[\expect_Y[L(h(\ts Y1{T}),h'(\ts Y1{T}))] - \expect_Y[L(h(\ts Y1{T-1}),h'(\ts Y1{T-1}))]\Big] \\
               &+ \sup_{h,h' \in \mathcal H} \Big[\expect_Y[L(h(\ts Y1{T-1}),h'(\ts Y1{T-1}))] - \frac 1m \sum_{i=1}^m L(h(\tsi Yi1{T-1}),h'(\tsi Yi1{T-1}))\Big]
  \end{align*}
  by sub-additivity of the supremum. Now, define
  \begin{align*}
    \phi(\past) \triangleq& \sup_{h,h' \in \mathcal H}  \Big[\frac 1m \sum_{i=1}^m L(h(\tsi Yi1{T}),h'(\tsi Yi1{T})) - \expect_Y[L(h(\ts Y1{T}),h'(\ts Y1{T}))]\Big]
  \end{align*}
  \begin{align*}
    \psi(\tpast) \triangleq& \sup_{h,h' \in \mathcal H}  \Big[\expect_Y[L(h(\ts Y1{T-1}),h'(\ts Y1{T-1})) - \frac 1m \sum_{i=1}^m L(h(\tsi Yi1{T-1}),h'(\tsi Yi1{T-1}))]\Big].
  \end{align*}
  By definition of $\disc_e$, we have from the previous inequality
  \begin{equation*}
    \disc_s \le \disc_e + \phi(\past_1^T) + \psi(\past_1^{T-1}).
  \end{equation*}
  We now proceed to give a high-probability bound for $\phi$; the same reasoning will yield a bound for $\psi$. By sub-additivity of the max,
  \begin{align*}
      \phi(\past) \le& \sum_j \frac{|C_j|}m \sup_{h \in \mathcal H} \Big[\expect_Y[f(h, \ts Y1T)] - \frac 1 {|C_j|} \sum_{Y \in C_j} f(h, \ts Y1T)\Big]\\
      \le& \sum_j \frac{|C_j|}m \phi(C_j)
  \end{align*}
  and so by union bound, for $\epsilon > 0$
  \[\Pr(\phi(\past) > \epsilon) \le \sum_j \Pr(\phi(C_j) > \epsilon ).\]
  
  Let $\epsilon > \max_j \expect[\phi(\widetilde C_j)]$ and set $\epsilon_j = \epsilon - \expect[\phi(\widetilde C_j)]$.

  Define for time series $Y(i),Y(j)$ the mixing coefficient
  \begin{equation*}
    \bar\beta(i,j) = \|\Pr(\tsi Yi1T, \tsi Yj1T) - \Pr(\tsi Yi1T)\Pr(\tsi Yj1T)\|_{TV}
  \end{equation*}
  where we also extend the usual notation to $\bar \beta(C_j)$.
  \begin{align*}
      \Pr\left(\phi(C_j) > \epsilon \right) =& \Pr\Big(\phi(C_j) - \expect[\phi(\widetilde C_j)] \hiderel>  \epsilon_j\Big)\\
      \stackrel{(a)}{\le}& \Pr\Big(\phi(\widetilde C_j) - \expect[\phi(\widetilde C_j)] >\epsilon_j\Big) + {(|I_j|-1)\bar\beta(I_j)}\\
      \stackrel{(b)}{\le}& e^{-2c\epsilon_j^2} + (|I_j|-1)\bar\beta(I_j),
  \end{align*}
  where (a) follows by applying Prop.~\ref{prop:yu} to the indicator function of the event $\Pr(\phi(C_{j}) - \expect[\phi(\widetilde C_{j}) ] \ge \epsilon)$, and (b) is a direct application of McDiarmid's inequality to $\phi(\widetilde C_j) -  \expect[\phi(\widetilde C_j)]$.
  
  Hence, by summing over $j$ we obtain
  \begin{align*}
      \Pr\left(\phi(\past) > \epsilon \right) \le& k e^{-2\min_j|C_j|(\epsilon-\max_j \expect[\phi(\widetilde C_j)])^2} + \sum_j (|I_j|-1)\bar\beta(I_j)
  \end{align*}
  and similarly
  \begin{align*}
    \Pr\left(\psi(\tpast) > \epsilon\right) \le& k e^{-2\min_j|C'_j|(\epsilon-\max_j \expect[\psi(\widetilde C'_j)])^2} + \sum_j (|I_j|-1)\bar\beta'(I_j),
  \end{align*}
  which finally yields
  \begin{align*}
      \Pr&(\disc_s - \disc_e > \epsilon) \le \Pr(\phi(\past) \hiderel> \epsilon) + \Pr(\psi(\tpast) \hiderel> \epsilon)\\
    &\le 2k \exp(-2c(\epsilon-\max (\max_j \expect[\phi(\widetilde C_j)], \max_j \expect[\psi(\widetilde C'_j)]))^2) + \sum_j (|I_j|-1)[\bar\beta(I_j) + \bar\beta'(I'_j)],
  \end{align*}
  where we recall that we write $c = \min_j |C_j|$. We invert the previous equation by setting
  \begin{align*}
    \epsilon =& \max (\max_j \expect[\phi(\widetilde C_j)], \max_j \expect[\psi(\widetilde C'_j)]) + \sqrt{\frac 1{2c}\log \frac{2k}{\delta - \sum_j (|I_j|-1)[\beta(I_j)+\beta(I'_j)]}},
  \end{align*}
  yielding with probability $1-\delta$,
  \begin{align*}
    \disc_s \le& \disc_e + \max (\max_j \expect[\phi(\widetilde C_j)], \max_j \expect[\psi(\widetilde C'_j)]) + \sqrt{\frac 1{2c}\log \frac{2k}{\delta - \sum_j (|I_j|-1)[\bar\beta(I_j)+\bar\beta'(I_j)]}}.
  \end{align*}
  
  We now bound $\expect[\phi(\widetilde C_j)]$ by $\radem_{|C_j|}(\widetilde C_j)$. A similar argument yields the bound for $\psi$. By definition, we have
  \begin{align*}
      \expect[\phi(\widetilde C_j)] &= \expect\Big[\sup_{h \in \mathcal H} \frac 1{|C_j|} \sum_{Z \in \widetilde C_j} f(h, \tsi Yi1T) - \expect_Y[f(h, \ts Y1T)]\Big]\\
     & = \frac 1{|C_j|} \expect\Big[\sup_{h \in \mathcal H}  \sum_{Z \in \widetilde C_j} \underbrace{f(h, \tsi Yi1T)\! -\! \expect_Y[f(h, \ts Y1T)}_{g(h, \tsi Yi1T)}]\Big]\\
      &= \frac 1{|C_j|} \expect\Big[\sup_{h \in \mathcal H}  \sum_{Z \in \widetilde C_j} g(h, \tsi Yi1T)\Big]
  \end{align*}
  Standard symmetrization arguments as those used for the proof of the famous result by~\cite{koltchinskii2002}, which hold also when data is drawn independently but not identically at random, yield
  \[\expect[\phi(\widetilde C_j)] \le \radem_{|C_j|}(\widetilde C_j).\]
  The same argument yields for $\psi$
  \[\expect[\psi(\widetilde C'_j)] \le \radem_{|C_j|}(\widetilde C'_j).\]
  To conclude our proof, it only remains to prove the bound
  \begin{align*}
    \bar \beta(i,j) \le& \hbeta(i,j) + \expect_{\tpast}\Big[\text{Cov}\Big(\Pr(Y_T(i) \hiderel \mid \tpast), \Pr(Y_T(j) \hiderel \mid \tpast)\Big)\Big]
  \end{align*}

  Let $Y(i),Y(j)$ be two time series, and write $X_i=\expect[\Pr(\tsi Yi1T)\hiderel\mid \tpast]$. Then the following bound holds
  \begin{align*}
      \bar\beta(i,j) =& \|\Pr(\tsi Yi1T, \tsi Yj1T) - \Pr(\tsi Yi1T)\Pr(\tsi Yj1T)\|_{TV}\\
       =& \|\expect[\Pr(\tsi Yi1T, \tsi Yj1T) \hiderel\mid \tpast] - \expect[X_i]\expect[X_j]\|_{TV}\\
       =& \|\expect[\Pr(\tsi Yi1T, \tsi Yj1T) \hiderel\mid \past_1^{T-1}] - \expect[X_i,X_j] - \expect[\text{Cov}(X_i,X_j)]\|_{TV}\\
       \le& \hbeta(i,j) + \expect_{\tpast}[\text{Cov}(X_i,X_j)],
  \end{align*}
  which is the desired inequality.
\end{proof}

We now show two useful lemmas for various specific cases of time series and hypothesis spaces.

\begin{repproposition}{prop:disc_stat}
  If $Y(i)$ is stationary for all $1 \le i \le m$, and $\mathcal H$ is a hypothesis space such that $h \in \mathcal H : \mathcal Y^{T-1} \to \mathcal Y$ (i.e. the hypotheses only consider the last $T-1$ values of $Y$), then $\disc_e =0$.
\end{repproposition}
\begin{proof}
  Let $h, h' \in \mathcal H$. For stationary $Y(i)$, we have $\Pr(\tsi Yi1T) = \Pr(\tsi Yi2T)$, and so
  \begin{equation*}
    \expect[L(h(\ts Y2T), h'(\ts Y2T))] - \expect[L(h(\ts Y1{T-1}),h'(\ts Y1{T-1}))] = 0
  \end{equation*}
  and so taking the supremum over $h,h'$ yields the desired result.
\end{proof}

\begin{repproposition}{prop:disc_cov_stat}
  If $Y(i)$ is covariance stationary for all $1 \le i \le m$, $L$ is the squared loss, and $\mathcal H$ is a linear hypothesis space $\{x \to w \cdot x \mid \|w\| \in \mathbb R^p \le \Lambda\}$, then $\disc_e=0$.
\end{repproposition}
\begin{proof}
  Recall that a time series $Y$ is covariance stationary if $\expect_Y [Y_t]$ does not depend on $t$ and $\expect_Y[Y_t Y_s] = f(t-s)$ for some function $f$.

  Let now $(h, h') \in \mathcal H\equiv (w, w') \in \mathbb R^p$. We write $\Sigma = \Sigma_2^T(Y)=\Sigma_1^T(Y)$ the covariance matrix of $Y$ where the equality follows from covariance stationarity. Without loss of generality, we consider $p=T-1$. Then,
  \begin{align*}
    \expect&[L(h(\ts Y2T), h'(\ts Y2T))] - \expect[L(h(\ts Y1{T-1}),h'(\ts Y1{T-1}))]
    \\&= \expect[((w-w')^\top \Sigma_2^T(Y) (w-w')] - \expect[((w-w')^\top \Sigma_1^{T-1}(Y) (w-w')] \\&= 0.
  \end{align*}
  Taking the supremum over $h,h'$ yields the desired result.
\end{proof}

\begin{repproposition}{prop:disc_period}
  If the $Y(i)$ are periodic of period $p$ and the observed starting time of each $Y(i)$ is distributed uniformly at random in $[p]$, then $\disc_e=0$.
\end{repproposition}
\begin{proof}
  This proof is similar to the stationary case: indeed, we can write $\Pr(\tsi Yi1{T-1}) = \frac 1p \Pr(Y(i))$ due to the uniform distribution on starting times. Then, by the same reasoning, we have also \[\Pr(\tsi Yi2{T}) = \frac 1p \Pr(Y(i)) = \Pr(\tsi Yi1{T-1}),\] from which the result follows.
\end{proof}

\section{Generalization bounds}
\label{app:dep-proof}

\begin{prop}{\emph{\citet[Corollary 2.7]{yu94}}}.
  \label{prop:yu}
  Let $f$ be a real-valued Borel measurable function such that $0 \le f \le 1$. Then, we have the following guarantee:
  \[\left|\expect[f(\widetilde C)] - \expect[f(C)]\right| \le (|C|-1)\beta,\]
  where $\beta$ is the total variation distance between
  joint distributions of $C$ and $\widetilde{C}$.
\end{prop}

\begin{reptheorem}{thm:dep}
  Let $\mathcal H$ be a hypothesis space, and $h \in \mathcal H$. Let $C_1, \ldots, C_k$ form a partition of the training input $\past_1^T$, and consider that the loss function $L$ is bounded by 1. Then, we have for $\delta > 0$, with probability $1-\delta$,
  \begin{align*}
    \hphi(h) \leq& \disc + \max_j \left[\radem_{|C_j|}(\widetilde C_j \hiderel\mid \past) \right] + \frac 1 {\sqrt {2\min_j |I_j|}} \sqrt{\log\left(\frac{k}{\delta - \sum_j (|I_j|-1)\hbeta(I_j)}\right)}.
  \end{align*}
\end{reptheorem}

For ease of notation, we write 
\begin{align*}
    \phi(\past) =& \sup_{h \in \mathcal H} \error(h \hiderel\mid \tpast) - \widehat \error(h, \past)\\
     =& \sup_{h \in \mathcal H} \frac 1m \sum_{i=1}^m \expect[f(h, \tsi Yi1T) \hiderel\mid \tpast] - \frac 1{m}\sum_{i=1}^m f(h, \tsi Yi1T).
\end{align*}
We begin by proving the following lemma.
\begin{lemma}
  \label{lem:mcd}
  Let $\bar \past$ be equal to $\past$ on all time series except for the last, where we have $\bar Y(m) = Y(m)$ at all times except for time $t=T$. Then
  \begin{equation*}
  \left|\phi(\past) - \phi(\bar \past)\right| \le \frac 1m
\end{equation*}
\end{lemma}
\begin{proof}
  Fix $h^* \in \mathcal H$. Then,
  \begin{align*}
    \error&(h^* \hiderel\mid \tpast) - \widehat \error(h^*, \past) - \sup_{h \in \mathcal H} \Big[\error(h \hiderel\mid \bar\tpast) - \widehat \error(h, \bar{\past})\Big] \\
    &\le \error(h^*\hiderel\mid \tpast) - \widehat L(h^*, \past) - \Big[\error(h^*, \hiderel\mid \bar\tpast) - \widehat \error(h^*, \bar{\past})\Big]\\ 
     &\stackrel{(a)}{\le} \widehat \error(h^*, \bar\past) - \widehat \error(h^*, \past)\\
     &\le \frac 1m \Big[f(h^*, \tsi {\bar Y}m1T) -  f(h^*, \tsi Ym1T)\Big] \hiderel \le \frac 1m.
 \end{align*}
 where (a) follows from the fact that $\tpast=\bar\tpast$ and the last inequality follows from the fact that $f$ is bounded by 1.

 By taking the supremum over $h^*$, the previous calculations show that $\phi(\past) - \phi({\bar \past}) \le 1/m$; by symmetry, we obtain $\phi({\bar \past}) - \phi({\past}) \le 1/m$ which proves the lemma.
\end{proof}

We now prove the main theorem.
\begin{proof}
Observe that the following bounds holds
  \begin{align*}
    \hphi(\past) =& \error(h\hiderel\mid \past) - \widehat \error(h, \past) \\
    \le& \sup_{h \in \mathcal H} \Big[\error(h\hiderel\mid \past) - \error(h\hiderel\mid \tpast)\Big] + {\sup_{h \in \mathcal H} \Big[\error(h\hiderel\mid \tpast) - \widehat \error(h, \past)\Big]}.
  \end{align*}
  and so 
  \begin{equation*}
    \hphi(\past) \hiderel- \disc \le \underbrace{\sup_{h \in \mathcal H}  \error(h, \hiderel\mid \tpast) - \widehat \error(h, \past)}_{\phi(\past)}.
  \end{equation*}
  Define $M = \max_j \expect[\phi(\widetilde C_j) \mid \widetilde \tpast]$. Then,
  \begin{align}
    \label{eq:main}
    \Pr&\Big(\hphi(\past) \hiderel- \disc - M \hiderel> \epsilon \hiderel\mid \tpast\Big) \le \Pr(\phi(\past) - M \hiderel> \epsilon \hiderel \mid \tpast).
  \end{align}

  By sub-additivity of the supremum, we have
  \begin{equation*}
    \phi(\past) - M \le \sum_{j} \frac {|C_j|}{m} {\sup_{h \in \mathcal H} \Big[\error(h\hiderel\mid \past) - \widehat \error(h, C_{j}) - M\Big]}
  \end{equation*}
  and so by union bound, 
  \begin{equation*}
    \Pr(\phi(\past) - M \ge \epsilon \hiderel\mid \tpast) \le \sum_{j} \Pr(\phi(C_{j}) - M \hiderel\ge \epsilon \hiderel\mid \tpast).
  \end{equation*}

  By definition of $M$, 
  \begin{align*}
    \Pr\Big(\phi(C_{j}) - M \hiderel\ge \epsilon \hiderel\mid \tpast\Big) &\le \Pr(\phi(C_{j}) - \expect[\phi(\widetilde C_{j}) \hiderel\mid \widetilde \tpast] \ge \epsilon \hiderel\mid \tpast)\\
     & \stackrel{(a)}{\le} \Pr(\phi(\widetilde C_{j}) - \expect[\phi(\widetilde C_{j}) \hiderel\mid \widetilde \tpast] \ge \epsilon \hiderel\mid \tpast) + (|I_{j}|-1) \hbeta(I_{j} \hiderel\mid \tpast)\\
      &\stackrel{(b)}{\le} e^{-2|C_{j}|\epsilon^2} + (|I_{j}|-1) \hbeta(I_{j} \hiderel\mid \tpast).
  \end{align*}
  where (a) follows by applying Prop.~\ref{prop:yu} to the indicator function of the event $\Pr(\phi(C_{j}) - \expect[\phi(\widetilde C_{j}) \hiderel\mid \widetilde \tpast] \ge \epsilon)$, and (b) is a direct application of McDiarmid's inequality, following Lemma~\ref{lem:mcd}. The notation $\hbeta(I_j \mid \tpast)$ indicates the total variation distance between the joint distributions of $C_j$ and $\widetilde C_j$ conditioned on $\tpast$. In particular, we have $\expect_{\tpast} \hbeta(C_j \mid \tpast) = \hbeta(C_j)$.

  Finally, taking the expectation of the previous term over all possible $\tpast$ values and summing over $j$, we obtain
  \begin{align*}
    \Pr&(\error(h\hiderel\mid \past) - \widehat \error(h, \past)- \expect_{\widetilde C_{j'}}[\phi(\widetilde C_{j'}) \hiderel\mid \widetilde \past] \ge \epsilon)  \le \sum_j e^{-2|C_j|\epsilon^2} + \sum_j (|I_j|-1) \hbeta(I_j).
  \end{align*}
  Combining this bound with \eqref{eq:main}, we obtain
  \begin{align*}
    \Pr\Big(\hphi(\past) \hiderel- \disc - M \hiderel> \epsilon\Big)
    &\le \sum_j e^{-2|C_j|\epsilon^2} + \sum_j (|I_j|-1) \hbeta(I_j)\\
      &\le k e^{-2\min_j|C_j| \epsilon^2} + \sum_j (|I_j|-1) \hbeta(I_j)
  \end{align*}
  We invert the previous equation by choosing $\delta > \sum_j (|I_j|-1)\hbeta(I_j)$ and setting
  \[\epsilon = \sqrt{\frac {\log \frac{k}{\delta - \sum_j (|I_j|-1)\hbeta(I_j)}}{2\min_j |I_j|}},\]
  which yields that with probability $1-\delta$, we have
  \begin{equation*}
    \hphi(Z) \le M + \disc + \sqrt{\frac{\log\left(\frac{k}{\delta - \sum_j (|I_j|-1)\beta(I_j)}\right)}{2\min_j |I_j|}}.
  \end{equation*}
  To conclude our proof, it remains to show that \[M \le \radem_{|C_j|}(\widetilde C_j \hiderel\mid \widetilde \tpast).\]
  \begin{align*}
      \expect[\phi(\widetilde C_j) \hiderel\mid \widetilde \tpast] =& \expect \Big[\sup_{h \in \mathcal H}  \error(h \hiderel\mid \widetilde \tpast) - \frac 1{|C_j|} \sum_{i=1}^m f(h, \tsi {\widetilde Y}i1T) \hiderel\mid \widetilde \tpast\Big]\\
\      =& \frac 1{|C_j|} \expect \Big[\sup_{h \in \mathcal H} \sum_{\ts {\widetilde Y} 1T \in \widetilde C_j} \expect [f(h, \ts {\widetilde Y}1T) \hiderel\mid \widetilde \tpast] - f(h, \tsi {\widetilde Y}i1T)\hiderel\mid \widetilde \tpast\Big]\\
      \le& \frac 1{| C_j|} \expect \Big[\sup_{h \in \mathcal H} \sum_{\ts {\widetilde Y} 1T \in \widetilde C_j}g(h,\tsi {\widetilde Y}i1T)\hiderel\mid \widetilde \tpast\Big]
  \end{align*}
  where we've defined \[g(h, \tsi {\widetilde Y}i1T) \triangleq \expect[f(h, \tsi {\widetilde Y}i1T) \hiderel\mid \widetilde\tpast] - f(h, \tsi {\widetilde Y}i1T).\]

  Similar arguments to those used at the end of Appendix~\ref{app:sym-bound} yield the desired result, which concludes the proof of Theorem~\ref{thm:dep}.
\end{proof}

\section{Generalization bounds for local models}
\begin{reptheorem}{th:local}
  Let $h=(h_1, \ldots, h_m)$ where each $h_i$ is a hypothesis learned via a local method to predict the univariate time series $Z_i$. For $\delta > 0$ and any $\alpha > 0$, we have w.p. with $1-\delta$
  \begin{align*}
  \nhphi(\train) \le& \frac 1m\sum_i \Delta(Y(i)) + 2 \alpha + \sqrt{\frac 2T \log \frac{m \max_i (\expect_{v \sim T(Y(i))}[\mathcal N_1(\alpha, \mathcal F, v)])}{\delta}}
  \end{align*}
\end{reptheorem}
\begin{proof}
  Write
  \begin{align*}
    \Phi(\tsi Yi1T) =& \sup_{h \in \mathcal H}\expect [f(h, \ts Y1{T+1}) \hiderel\mid \ts Y1T] - \frac 1T\sum_{t=1}^T f(h, \tsi Yi{t}{t+T}).
  \end{align*}
  By~\citep[Theorem~1]{kuznetsov2015}, we have that for $\epsilon > 0$, and $1 \le i \le m$, 
  \begin{align*}
    \Pr(\Phi(\tsi Yi1T - \Delta(Y(i)) > \epsilon) \le& \expect_{v \sim T(p)}[\mathcal N_1(\alpha, \mathcal F, v)]  \exp{\Big(-\frac{T(\epsilon - 2 \alpha)^2}{2}\Big)}.
  \end{align*}
  By union bound,
  \begin{align*}
    \Pr(\frac 1m \sum_i \Phi(\tsi Yi1T) - \Delta(Y(i)) > \epsilon)&\le m \max_i (\expect_{v \sim T(Y(i))}[\mathcal N_1(\alpha, \mathcal F, v)]) \exp{\Big(-\frac{T(\epsilon -2\alpha)^2}{2}\Big)}
  \end{align*}
  We invert the previous equation by letting
  \begin{equation*}
    \epsilon = 2 \alpha + \sqrt {\frac 2T \log \frac {m \max_i (\expect_{v \sim T(Y(i))}[\mathcal N_1(\alpha, \mathcal F, v)])} \delta}.
  \end{equation*}
  which yields the desired result.
\end{proof}

\section{Analysis of expected mixing coefficients}
\label{app:hierarchy}
\begin{replemma}{lemma:gaussian}
    Two AR processes $Y(i),Y(j)$ generated by~\eqref{eq:gaussian} such that $\sigma = \text{Cov}(Y(i),Y(j)) \le \sigma_0 < 1$ verify $\hbeta(i,j) = \max \left(\frac 3{2(1-\sigma_0^2)}, \frac 1 {1-2\sigma_0}\right)\sigma$.
\end{replemma}
\begin{proof}
  For simplicity, we write $U=Y(i)$ and $V = Y(j)$.

  Write
  \begin{align*}
      \beta =&\| P(U_{T} | \tpast ) P(V_{T} | \tpast ) - P(U_{T}, V_{T} | \tpast )\|_{TV}\\
      =& {\sup_{u,v} \left|P(U_T\!=\!u)P(V_T\!=\!v)-P(U_T=u,V_T=v)\right|}\\
      =& \sup_{u,v} \Big|{P(U_T=u \mid \ts{U}0{T-1})P(V_T = v \mid \ts v0{T-1})} - {P(U_T=u,V_T=v \mid \ts v0{T-1}, \ts {u}0{T-1})}\Big|\\
      =& \sup_{u,v} \Big|\Big[{P(u,v \mid \ts U0{T-1}, \ts {V}0{T-1})} + f(\sigma, \delta, \epsilon)\Big]  -{P(u,v \mid \ts U0{T-1}, \ts {V}0{T-1})}\Big|
  \end{align*}
  where we've written $\delta = u - \Theta_i( \ts U0{T-1})$ (and $\epsilon$ similarly for $v$), and we've defined
  \begin{align*}
      f(\sigma,\delta,\epsilon) &= {P(u | \ts{U}0{T-1})P(v | \ts V0{T-1})} - {P(u,v | \ts U0{T-1}, \ts {V}0{T-1})}\\
      &= e^{-\frac 12(\delta^2+\epsilon^2)} - \frac 1{1-\sigma^2}e^{-\frac 12 \frac 1{1-\sigma^2} (\delta^2 + \epsilon^2 - 2\sigma\epsilon\delta)}.
  \end{align*}
  Assuming we can bound $f(\sigma, \delta, \epsilon)$ by a function $g(\sigma)$ independent of $\delta, \epsilon$, we can then derive a bound on $\beta$.

  Let $x=\sqrt{\delta^2+\epsilon^2}$ be a measure of how far the AR process noises lie from their mean $\mu=0$. Using the inequality \[|\delta\epsilon| \le \delta^2 + \epsilon^2,\] we proceed to bound $|f(\sigma, \delta, \epsilon)|$ by bounding $f$ and $-f$.
  
  \begin{align*}
    f(\sigma,\delta,\epsilon) &\le e^{-\frac 12(\delta^2+\epsilon^2)} - e^{-\frac 12 \frac {1}{1-\sigma^2} (\delta^2 + \epsilon^2 + 2\sigma|\delta\epsilon|)}  \\
                              &\le e^{-\frac 12x^2}-e^{-\frac 12 \frac 1{1-\sigma^2} (1+2\sigma) x^2}   \\
                              &\le e^{-\frac 12x^2}\Big(1-e^{-\frac 12 \frac{2\sigma+\sigma^2}{1-\sigma^2} x^2}\Big)
  \end{align*}
  Using the inequality $1-x \le e^{-x}$, it then follows that
  \begin{align}
    f(\sigma,\delta,\epsilon) &\le e^{-\frac 12x^2}(1-(1-\tfrac 12 \tfrac{2\sigma+\sigma^2}{1-\sigma^2} x^2)) \nonumber \\
                              &\le \tfrac 12 \tfrac 3{1- \sigma^2}\sigma x^2 e^{-\frac 12 x^2} \nonumber \\
                              &\stackrel{(a)}{\le} \tfrac 3 {e(1-\sigma^2)} \sigma\label{eq:le}
  \end{align}
  where inequality $(a)$ follows from the fact that $y \to y e^{-y}$ is bounded by $1/e$.

  Similarly, we now bound $-f$: 
  \begin{align*}
    - f(\sigma,\delta,\epsilon) &\le \frac 1{1-\sigma^2} e^{-\frac 12 \frac 1{1-\sigma^2} (\delta^2+\epsilon^2 - 2\sigma|\epsilon\delta|)} - e^{-\frac 12(\delta^2+\epsilon^2)}  \\
                                &\le \frac 1{1-\sigma^2} e^{-\frac 12 \frac {1-2\sigma}{1-\sigma^2} x^2} - e^{-\frac 12x^2}  \\
                                &\le \frac 1{1-\sigma^2} e^{-\frac 12 (1-2\sigma) x^2}  - e^{-\frac 12x^2}.
  \end{align*}
  One shows easily that this last function reaches its maximum for $x_0^2 = \frac 1 \sigma \log(\frac{1-\sigma^2}{1-2\sigma})$, at which point it verifies
  \begin{equation}
    \label{eq:ge} 
    - f(\sigma, x_0) = \frac {2\sigma} {1-2\sigma}e^{-\frac 1 {2\sigma} \log(\frac{1-\sigma^2}{1-2\sigma})} \le \frac{2\sigma}{1-2\sigma}
  \end{equation}
  Putting ~\eqref{eq:le} and ~\eqref{eq:ge} together, we obtain
  \begin{align*}
      |f(\sigma,\delta,\epsilon)| &\le \sigma  \max \left(\frac{3}{e(1-\sigma^2)}, \frac 1{1-2\sigma}\right)\\ 
      &\le \max \left(\frac 3{2(1-\sigma_0^2)}, \frac 1 {1-2\sigma_0}\right)\sigma
  \end{align*}

  Taking the expectation over all possible realizations of $\tpast$ yields the desired result.
\end{proof}

\begin{proof}
Recall that \past contains $m'=mT$ examples, which we denote $Y_{t-p}^t(i)$ for $1 \le i \le m$ and $1 \le t \le T$ (when $t-p < 0$, we truncate the time series approprietly).
We define
\begin{equation*}
  \hybloss(h \mid \past) = \frac 1{m}\sum_{i=1}^{m}\expect[L(h(Y_{T-p+1}^{T}(i)),Y_{T+1}(i)) \hiderel\mid \past]
\end{equation*}
\begin{equation*}
  \hybloss(h \mid \tpast) = \frac 1{m}\sum_{i=1}^{m}\frac 1T\sum_{t=1}^T\expect[L(h(Y_{t-p}^{t-1}(i)),Y_t(i)) \hiderel\mid \tpast]
\end{equation*}
\begin{equation*}
  \widehat\hybloss(h) = \frac 1{m}\sum_{i=1}^{m}\frac 1T\sum_{t=1}^TL(h(Y_{t-p}^{t-1}(i)),Y_{t}(i))
\end{equation*}
where we note that here $\tpast$ indicates each of the $mT$ training samples excluding their last time point.

Observe that the following chain of inequalities holds:
  \begin{align*}
      \hybphi(\past) =& \sup_{h \in \mathcal H} \hybloss(h \hiderel\mid \past) - \widehat \hybloss(h)\\
      \le& {\sup_{h \in \mathcal H} \Big[\hybloss(h\hiderel\mid \past) - \hybloss(h\hiderel\mid \tpast)\Big]} + {\sup_{h \in \mathcal H} \Big[\hybloss(h\hiderel\mid \tpast) - \widehat \hybloss(h, \past)\Big]}\\
    \le& \frac 1T \sum_{t=1}^T \sup_{h \in \mathcal H} \Big[\hybloss(h\hiderel\mid \past) - \frac 1m\sum_{i=1}^m\expect_{\dist}[L(h(Y_{t-p}^{t-1}(i)),Y_t(i)) \hiderel\mid \tpast]\Big] \\
                      &+ \sup_{h \in \mathcal H} \Big[\hybloss(h\hiderel\mid \tpast) - \widehat \hybloss(h, \past)\Big].
  \end{align*}
  and so 
  \begin{equation*}
    \hybphi(\past) \hiderel- \frac 1T \sum_t \disc_t \le\underbrace{\sup_{h \in \mathcal H}  \hybloss(h, \hiderel\mid {\tpast}) - \widehat \hybloss(h, \past)}_{\phi(\past)}.
  \end{equation*}
  Then, following the exact same reasoning as above for $\hphi$ shows that for $\delta > 0$, we have with probability $1-\delta/2$
  \begin{align*}
    \hybphi(\past) &\le \underbrace{\max_j \widehat \radem_{\widetilde C_j}(\mathcal F) + \frac 1T \sum_t \disc_t + \sqrt{\frac{\log\left(\frac{2k}{\delta - \sum_j (|I_j|-1)\beta(I_j)}\right)}{2\min_j |I_j|}}}_{B_1}
  \end{align*}

  However, upper bounding $\hybphi$ can also be approached using the same techniques as~\citet{kuznetsov2015}, which we now describe. Let $\alpha > 0$. For a given $h$, computing $\hybloss(h, \past)$ is similar in expectation to running $h$ on each of the $m$ time series, yielding for each time series $Y_{T-p+1}^{T}(i)$ the bound
  \begin{align*}
    \expect&[L(h(Y_{T-p+1}^{T}(i)),Y_{T+1}(i)) \hiderel\mid \past] \\
           &\le \frac 1T\sum_{t=1}^TL(h(Y_{t-p}^{t-1}(i)),Y_{t}(i)) + \Delta(\past_i) +2\alpha + \sqrt{\frac 2T \log \frac{\max_i \expect_{v \sim T(\past_i)}[\mathcal N_i(\alpha, \mathcal F, v)]}{\delta}}
  \end{align*}
  and so by union bound, as above, we obtain with probability $1-\delta/2$
  \begin{align*}
    \hybphi(\past) \le&\frac 1m \sum \Delta(\past_i) + 2\alpha + \sqrt{\frac 2T \log \frac{2m\max_i \expect_{v \sim T(\past_i)}[\mathcal N_i(\alpha, \mathcal F, v)]}{\delta}}\\
    \le& B_2
  \end{align*}

  We conclude by a final union bound on the event $\{\hybphi(\past) \ge B_1 \cup \hybphi(\past) \ge B_2\}$, we obtain with probability $1-\delta$,
  \begin{equation*}
        \hybphi(\past) \le \min (B_1,B_2)
  \end{equation*}
\end{proof}
  

\end{document}